\documentclass[10pt, conference, compsocconf]{IEEEtran}
\ifCLASSINFOpdf
\else
\fi
\usepackage{amsmath}
\usepackage{graphicx}
\usepackage{subfig}
\usepackage{algorithm}
\usepackage{algorithmic}
\usepackage{verbatim}
\newtheorem{proposition}{Proposition}
\newtheorem{fact}{Fact}
\newtheorem{definition}{Definition}
\newtheorem{lemma}{Lemma}
\newtheorem{example}{Example}

\begin{document}
\title{Online Anomaly Detection Systems \\ Using Incremental Commute Time}


% \author{Submitted for Blind Review}


\author{\IEEEauthorblockN{Nguyen Lu Dang Khoa\IEEEauthorrefmark{1}
and Sanjay Chawla\IEEEauthorrefmark{2}}
\IEEEauthorblockA{School of Information Technologies \\ University of Sydney \\
Sydney NSW 2006, Australia \\
\IEEEauthorrefmark{1}khoa@it.usyd.edu.au}
\IEEEauthorrefmark{2}sanjay.chawla@sydney.edu.au
}


\maketitle

\begin{abstract}
Commute Time Distance (CTD) is a random walk based metric on graphs. CTD has found widespread applications in many domains including personalized search, collaborative filtering and making search engines robust against manipulation. Our interest is inspired by the use of CTD as a metric for anomaly detection. It has been shown that CTD can be used to simultaneously identify both global and local anomalies. Here we propose an accurate and efficient approximation for computing the CTD in an incremental fashion in order to facilitate real-time applications. An online anomaly detection algorithm is designed where the CTD of each new arriving data point to any point in the current graph can be estimated in constant time ensuring a real-time response. Moreover, the proposed approach can also be applied in many other applications that utilize  commute time distance.
\end{abstract}


% \fi
\IEEEpeerreviewmaketitle

\section{Introduction}
Commute Time Distance (CTD) is a random walk based metric on graphs. The $CTD(i,j)$ between two nodes $i$ and $j$ is the {\it expected} number of steps a random walk starting at $i$ will take to reach $j$ for the first time and then return back to $i$. The fact that CTD is averaged over all paths (and not just the shortest path) makes it more robust to data perturbations.

CTD has found widespread applications in personalized search \cite{sarkar2008}, collaborative filtering \cite{brand2005,fouss2007} and making search engines robust against manipulation \cite{hopcroft2007}. Our interest is inspired by the use of CTD as a metric for anomaly detection. It has been shown that CTD can be used to simultaneously identify global, local and even collective anomalies in data \cite{khoa2010}.

More advanced measures generally require more expensive computation. Estimating CTD involves the eigen decomposition of the graph Laplacian matrix and consequently has $O(n^3)$ time complexity which is impractical for large graphs. Saerens, Pirotte and Fouss \cite{saerens2004b} used subspace approximation, and Khoa and Chawla \cite{khoa2010} used graph sampling to reduce the complexity. Sarkar and Moore \cite{sarkar2007} introduced a notion of truncated commute time and a pruning algorithm to find nearest neighbors in commute time. They empirically demonstrated achieving a near-linear running time as a function of graph size. Spielman and Srivastava \cite{spielman2008} have proposed a near-linear time algorithm for approximating pairwise CTD in $O(\log n)$ time based on random projections.

However, there are many applications in practice which require the computation of CTD in an online fashion. When a new data point arrives, the application needs to respond quickly without recomputing everything from scratch. The algorithms noted above all work in a batch fashion and have a high computation cost for online applications.

We are interested in the following scenario: a dataset $D$ is given from an underlying domain of interest. For example, data from a network traffic log or environment or climate change monitoring. A new data point arrives and we want to determine if it is an anomaly with respect to $D$ in CTD. Intuitively a new data point is an anomaly if it is {\it far away from its neighbors in CTD}.

\begin{example}
Consider the two graphs shown in Figure \ref{fig:CTD1}. While the shortest path distance between node 1 and 2
is the same in both graphs, $CTD(1,2)$ increases after node 5 is added. This property of CTD can be used
to great effect to detect both global and local outliers. However, the same property makes it challenging
to calculate the CTD in an incremental manner.

\begin{figure}[h]
  \centering
  \subfloat[A graph of 4 nodes]{\label{fig:exp1}\includegraphics[width=0.15\textwidth]{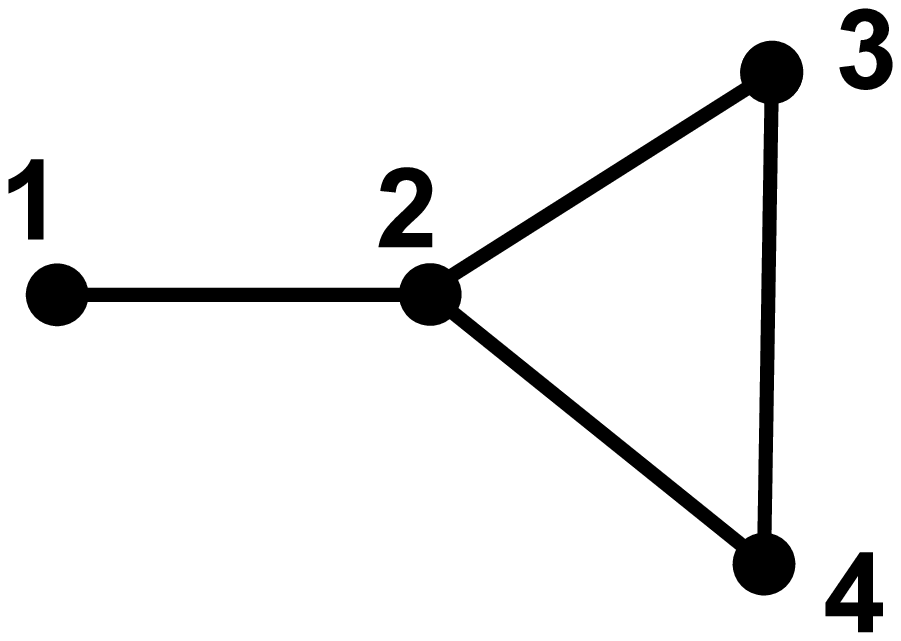}}
  \hspace{1cm}
  \subfloat[Adding node 5]{\label{fig:exp2}\includegraphics[width=0.15\textwidth]{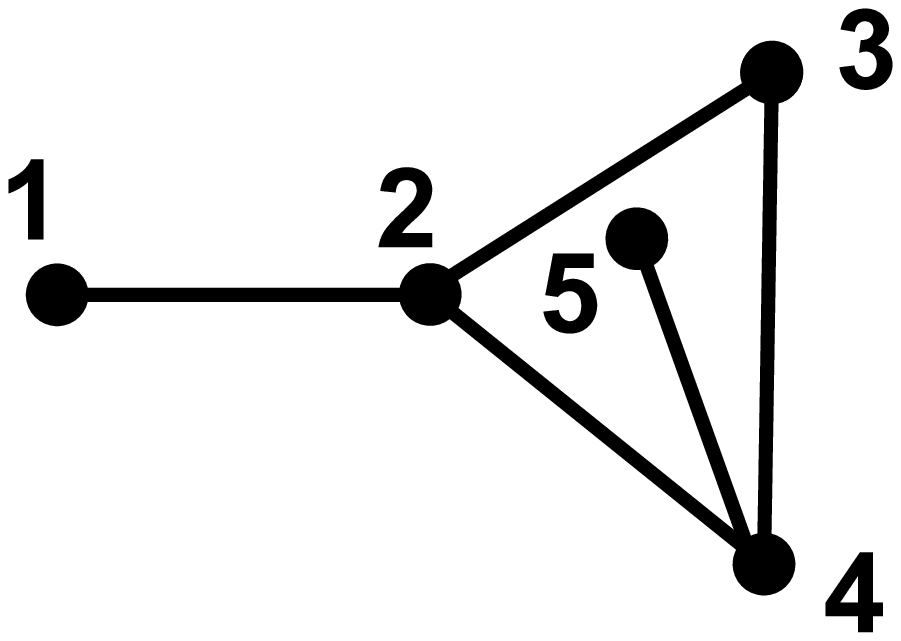}}
  \caption{$CTD(1,2)$ increases after addition of node 5 even though shortest path distance remains
unchanged. This property of CTD has many applications including anomaly detection.}
  \label{fig:CTD1}
\end{figure}
\end{example}

Here we propose and compare two methods for computing CTD in an incremental fashion. The first method is based on an incremental update of the eigen decomposition of a Laplacian matrix. The second method uses the recursive definition of CTD based on hitting time. To the best of our knowledge both these methods are novel and their comparison provides
revealing insights about CTD which are independent of the application domain.

The contributions of this paper are as follows:
\begin{itemize}
    \item We make use of the characteristics of random walk to estimate CTD incrementally in constant time. The same approach could be used for estimating the hitting time incrementally.
    \item We propose a provably fast method to incrementally update the eigenvalues and eigenvectors of the graph Laplacian matrix. This method can be integrated with any technique requiring a graph spectral computation, such as spectral clustering.
    \item We design an online algorithm for anomaly detection using incremental CTD. The technique is verified by experiments in synthetic and real datasets. The experiments show the effectiveness of the proposed methods in terms of accuracy and performance.
\end{itemize}

The remainder of the paper is organized as follows. Section \ref{chapter:CTD} reviews notation and concepts related to random walk and CTD, and a simple example to tie up all the definitions and ideas. In Sections \ref{chapter:incrementalLaplacian} and \ref{chapter:incrementalCD}, we present two methods to incrementally approximate the CTD. In Section \ref{iCTDalgorithm}, we introduce an online anomaly detection algorithm which uses incremental CTD. In Section \ref{chapter:expres}, we evaluate our approach using experiments on synthetic and real datasets. Sections \ref{chapter:related} covers related work. We conclude in Section \ref{chapter:conclusion} with a summary and a direction for future research.

\section{Commute Time Distance}
\label{chapter:CTD}
We provide a self-contained introduction to random walks with an emphasis on CTD. Assume we are given a connected undirected and weighted graph $G=(V,E,W)$.

\begin{definition} Let $i$ be a node in $G$ and $N(i)$ be its neighbors. The {\it degree} $d_{i}$ of a node $i$ is $\sum_{j \in N(i)}w_{ij}$. The {\it volume} $V_{G}$ of the graph is defined as $\sum_{i \in V}d_{i}$.
\end{definition}

\begin{definition} The transition matrix $M =(p_{ij})_{i,j \in V}$ of a random walk on $G$ is given by
\[
p_{ij} = \left\{\begin{array}{ll}\frac{w_{ij}}{d_{i}}, & \mbox{ if $(i,j) \in E$} \\0, & \mbox{ otherwise }
\end{array} \right.
\]
\end{definition}

\begin{definition}
Let $P_{0}$ be an initial distribution on $G$. Define $P_{t} = (M^{T})^{t}P_{0}$ for all $t \geq 0$. A distribution $P_{0}$ is {\it stationary} if $P_{1}=P_{0}$ \cite{lovasz1993}.
\end{definition}

\begin{fact} The distribution $P_{0}$ defined by $\pi(v) = \frac{d_{v}}{V_{G}}$ for all $v \in V$ is a {\it stationary distribution} \cite{lovasz1993}.
\end{fact}

\begin{definition} A random walk is {\it time-reversible} if for every pair of nodes $i,j \in V$, $\pi(i)p_{ij} = \pi(j)p_{ji}$ \cite{lovasz1993}.
\end{definition}

\begin{definition} The Hitting Time $h_{ij}$ is the expected number of steps that a random walk starting at $i$ will take before reaching $j$ for the first time.
\end{definition}

\begin{definition}The Hitting Time can be defined in terms of the recursion
\begin{displaymath}
h_{ij} = \left\{\begin{array}{ll}
1 + \sum_{l \in N(i)}p_{il}h_{lj} & \mbox{ if $ i \neq j$} \\
0 & \mbox{otherwise}
\end{array}\right.
\end{displaymath}
\end{definition}

\begin{definition} The Commute Time Distance $c_{ij}$  between two nodes $i$ and $j$ is given by $c_{ij} = h_{ij} + h_{ji}$.
\end{definition}

\begin{fact} CTD is a metric: (i) $c_{ii} = 0$, (ii) $c_{ij} = c_{ji}$ and (iii) $c_{ij} \leq c_{ik} + c_{kj}$ \cite{klein1993}.
\end{fact}

Remarkably, CTD can be expressed in terms of the Laplacian of $G$.
\begin{definition}Let $D$ be the diagonal degree matrix and $A$ be the adjacency matrix of $G$. The Laplacian of $G$ is the matrix $L = D - A$.
\end{definition}

\begin{fact}
\begin{enumerate}
\item
Let $e_{i}$ be the $V$ dimensional column vector with a 1 at location $i$ and zero elsewhere.
\item Let $(\lambda_{i}, v_{i})$ be the eigenpair of $L$ for all $i \in V$, i.e., $Lv_{i} = \lambda_{i}v_{i}$.
\item It is well known that $\lambda_{1}=0, v_{1} =(1,1,\ldots,1)^{\text{T}}$ and all $\lambda_{i} \geq 0$.
\item Assume $0 = \lambda_{1} \leq \lambda_{2} \ldots \leq \lambda_{|V|}$.
\item Then the pseudo-inverse of $L$ denoted by $L^{+}$ is
\[
L^{+} = \sum_{i=2}^{|V|}\frac{1}{\lambda_{i}}v_{i}v_{i}^{\text{T}}
\]
\end{enumerate}

\end{fact}
\begin{fact}
\begin{equation}
\label{equa3}
c_{ij} = V_{G}(l_{ii}^{+}+l_{jj}^{+}-2l_{ij}^{+}) = V_{G}(e_{i}-e_{j})^{\text{T}}L^{+}(e_{i}-e_{j})
\end{equation}
where $l_{ij}^{+}$ is the $(i,j)$ element of $L^{+}$ \cite{fouss2007}.
\end{fact}

\vspace{5pt}
{\bf Example:} Again, consider the graph $G$ shown in Figure \ref{fig:exp1} where all the edge weights equal to 1. The sum of the degree of nodes, $V_{G}=8$. We will calculate the commute time $c_{12}$ in two different ways:

\begin{enumerate}
\item  Using random walk: note that the expected number of steps for a random walk starting at node 1 and returning back to it is $\frac{V_{G}}{d_{1}} = \frac{8}{1} = 8$ \cite{lovasz1993}. But the walk from node 1 can only go to node 2 and
then return from node 2 to 1. Thus $c_{12} = 8$.
\item Using algebraic approach: the Laplacian matrix is
\[
L= \left(
\begin{array}{rrrr}
1 & -1 & 0 & 0 \\
-1 & 3 & -1 & -1 \\
0 & -1 & 2 & -1 \\
0 & -1 & -1 & 2
\end{array} \right)
\]
and the pseudo-inverse is
\[
L^{+} = \left(
\begin{array}{rrrr}
0.69 & -0.06 & -0.31 & -0.31 \\
-0.06 & 0.19 & -0.06 & -0.06 \\
-0.31 & -0.06 & 0.35 & 0.02 \\
-0.31 & -0.06 & 0.02 & 0.35  \\
\end{array} \right)
\]
Now $c_{12} = V_{G}(e_{1}-e_{2})^{\text{T}}L^{+}(e_{1}-e_{2})$ and \\
\begingroup
\everymath{\scriptstyle}
\tiny
\[
\left(\begin{array}{cccc} 1 & -1 & 0 & 0 \end{array}\right)
\left(\begin{array}{rrrr}
0.69 & -0.06 & -0.31 & -0.31 \\-0.06 & 0.19 & -0.06 & -0.06 \\
-0.31 & -0.06 & 0.35 & 0.02 \\-0.31 & -0.06 & 0.02 & 0.35  \end{array} \right)
\left(\begin{array}{r} 1 \\  -1 \\ 0 \\ 0 \\ \end{array} \right)
= 1
\]
\endgroup
\end{enumerate}
Thus $c_{12} = V_{G}\times 1 =8$.

Suppose we add a new node (labeled 5) to node 4 with a unit weight  as in Figure \ref{fig:exp2}. Then $c_{12}^{new}=V_{G}^{new}/d_{1}=10/1=10.$

The example in Figure \ref{fig:exp2} shows that by adding an edge, i.e. making the ``cluster'' which contains node 2 denser, $c_{12}$ increases. This shows that CTD between two nodes captures not only the distance between them (as measured by the edge weights) but also their neighborhood densities. For the proof of this claim, see \cite{khoa2010}. This property of CTD has been used to simultaneously discover global and local anomalies in data - an important problem in the anomaly detection literature.

In the above example, we exploited the specific topology (degree one node) of the graph to calculate CTD efficiently. This can only work for very specific instances. The general, more widely used  but slower approach for computing CTD is to use the Laplacian formula.  A key contribution of this paper is that for incremental computation of CTD we can use insights from this example to accurately and efficiently compute the CTD in much more general situations.

\section{Incremental Eigen Decomposition of Graph Laplacian}
\label{chapter:incrementalLaplacian}
In this section, we propose a method to incrementally update the eigensystem (eigenvalues and eigenvectors) of the Laplacian when a new node along with edges to its neighbors is added to the underlying graph. The unique feature of our approach as opposed to that of Ning et. al. \cite{ning2007} are (i) our emphasis is on handling the addition of a new node and the corresponding edges to its nearest neighbors as opposed to just weight updates on existing edges and (ii) simultaneous updating of all weight edges as opposed to one edge at a time.

\subsection{Iterative incremental update of the Laplacian eigensystem}
We propose an algorithm based on the following proposition to incrementally update the eigensystem $(\lambda, v)$ of the Laplacian $L$ when a new node $i$ is added to the graph. Suppose there are $k$ edges $e = (i,j) \in E_{n}$ with weight $w_{e}$ added to the graph from $i$. Denote $\Delta L$ and ($\Delta \lambda, \Delta v$) be changes of $L$ and $(\lambda, v)$ resulting from the addition of $i$. Note that the size of matrix $L$ and its eigenvector $v$ change as mentioned in the Appendix.

\begin{proposition}
\label{prop2}
The solution of the eigensystem
\[
(L + \Delta L)(v + \Delta v) = (\lambda + \Delta \lambda)(v + \Delta v)
\]
can be derived from the solution of the following set of simultaneous equations.
\begin{equation}
\label{equa7}
\Delta\lambda = % \frac{v^{\text{T}}\Delta L(v + \Delta v)}{1+v^{\text{T}}\Delta v}
                \frac{\sum_{e \in E_{n}}w_{e}[v(i)-v(j)][v(i)-v(j)+\Delta v(i)-\Delta v(j)]}{1 + v^{\text{T}}\Delta v}
\end{equation}
\begin{equation}
\label{equa10}
\Delta v = K^{-1}h
\end{equation}
where
\begin{equation}
\label{equa8}
K = L + \Delta L - (\lambda + \Delta \lambda)I,
\end{equation}
and
\begin{equation}
\label{equa9}
h = (\Delta \lambda I - \Delta L)v.
\end{equation}
\end{proposition}

For the proof, see Appendix.

Note that in general it is not practical to solve the system $\Delta v  = K^{-1}h$ at the arrival of each new data point $i$. In practice, as noted by Ning. et. al. \cite{ning2007},  we can set $\Delta v(k) =0$ for all components which are not $i$, its first or second order neighbors.

Denote $N_{i}=\{j|d(i,j) \leq 2\}$, where $d(i,j)$ is the shortest path between $i$ and $j$. Let $K_N$ be the  matrix derived from $K$ after removing columns which do not correspond to nodes in $N_{i}$, $v_N$ and $\Delta v_N$ be the vectors derived from $v$ and $\Delta v$ after removing elements which do not correspond to nodes in $N_{i}$. Since $K_N$ is not a square matrix, we obtain:
\begin{equation}
\label{equa11}
\Delta v = (K_N^{\text{T}}K_N)^{-1}K_N^{\text{T}}h,
\end{equation}
\begin{equation}
\label{equa4}
\Delta\lambda = \frac{\sum_{e \in E_{n}}w_{e}[v(i)-v(j)][v(i)-v(j)+\Delta v(i)-\Delta v(j)]}{1 + v_N^{\text{T}}\Delta v_N}
\end{equation}

Since $\Delta \lambda$ in Equation \ref{equa4} depends on the value of $\Delta v$ in Equation \ref{equa11} and vice versa, we can update the values of $\Delta \lambda$ and $\Delta v$ as follows. We initialize the values $\Delta v = 0$ to update the value of $\Delta \lambda$ and then using that to update $\Delta v$. The procedure is repeated until convergence.  Algorithm \ref{iterativealgorithm}
gives the details.

\begin{algorithm}[h!]
\caption{Incremental update eigenvalues and eigenvectors of Laplacian matrix $L$}
\label{iterativealgorithm}
\textbf{Input:} Laplacian matrix \emph{L}, its eigenvalues \emph{S} and eigenvectors \emph{V}, weights \emph{$w_{e}$} of all the new edges \\
\textbf{Output:} New eigenvalues \emph{$S_n$} and eigenvectors \emph{$V_n$} \\
\begin{algorithmic}[1]
\FOR {each eigenvalue and eigenvector}
\STATE Set $\Delta v=0$\\
\STATE Update $\Delta \lambda$ using Equation \ref{equa4} \\
\STATE Update $\Delta v$ using Equation \ref{equa11} \\
\STATE Repeat steps 3 and 4 until there is no significant change in $\Delta \lambda$ or until the loop reaches a maximum iterations \\
\STATE $v_n = v + \Delta v$, $\lambda_n = \lambda + \Delta \lambda$
\ENDFOR
\end{algorithmic}
\end{algorithm}

\section{Incremental Estimation of Commute Time Distance}
\label{chapter:incrementalCD}
In this section, we derive a new method for computing the CTD in an incremental fashion. This method uses the definition of CTD based on the hitting time. The basic intuition is to expand the hitting time recursion until the random walk has moved a few steps away from the new node and then use the {\it old} values. In
Section \ref{chapter:expres} we will show that this method results in remarkable agreement between the batch and online mode.

We deal with two cases shown in Figure \ref{fig:rank}.

\begin{figure}[h]
  \centering
  \subfloat[Rank 1]{\label{fig:rank1}\includegraphics[width=0.2\textwidth]{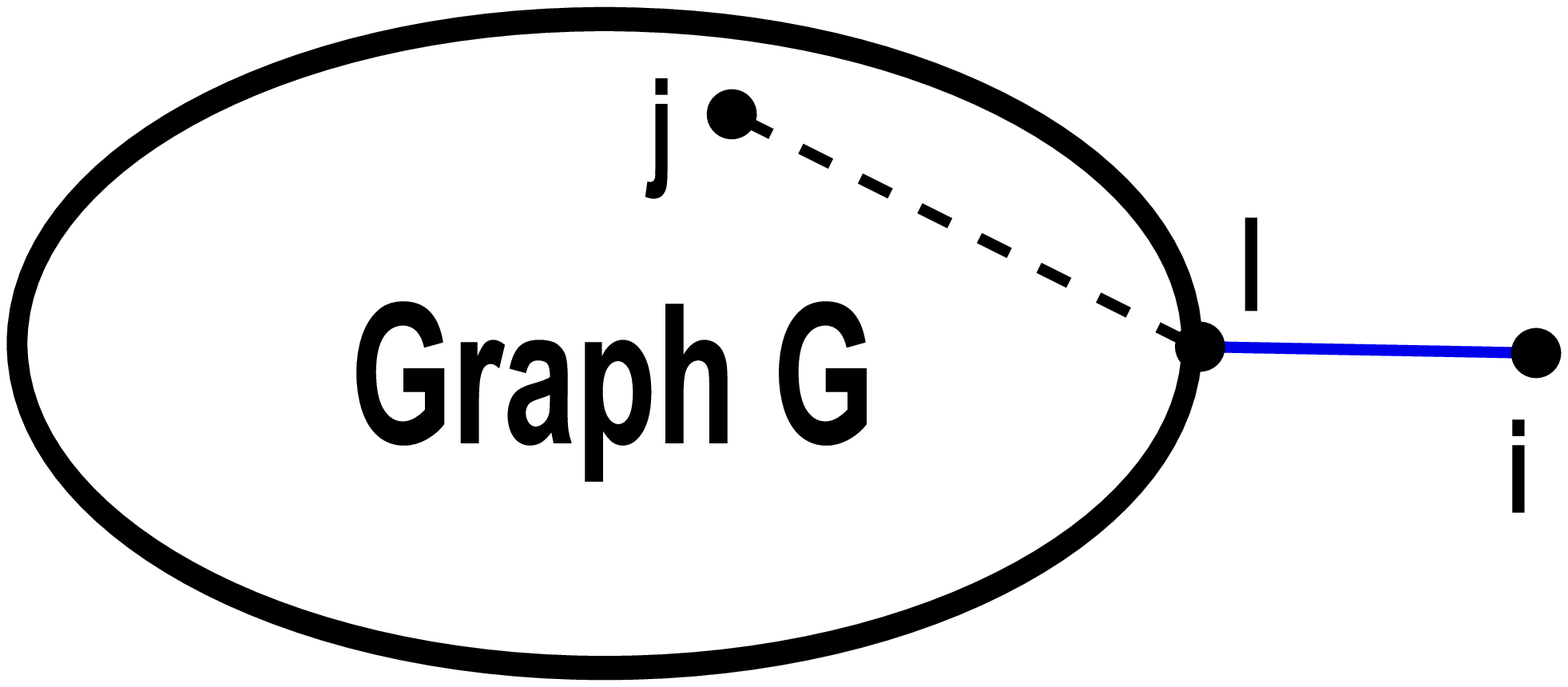}}
\hspace{+5pt}
  \subfloat[Rank k]{\label{fig:rankk}\includegraphics[width=0.2\textwidth]{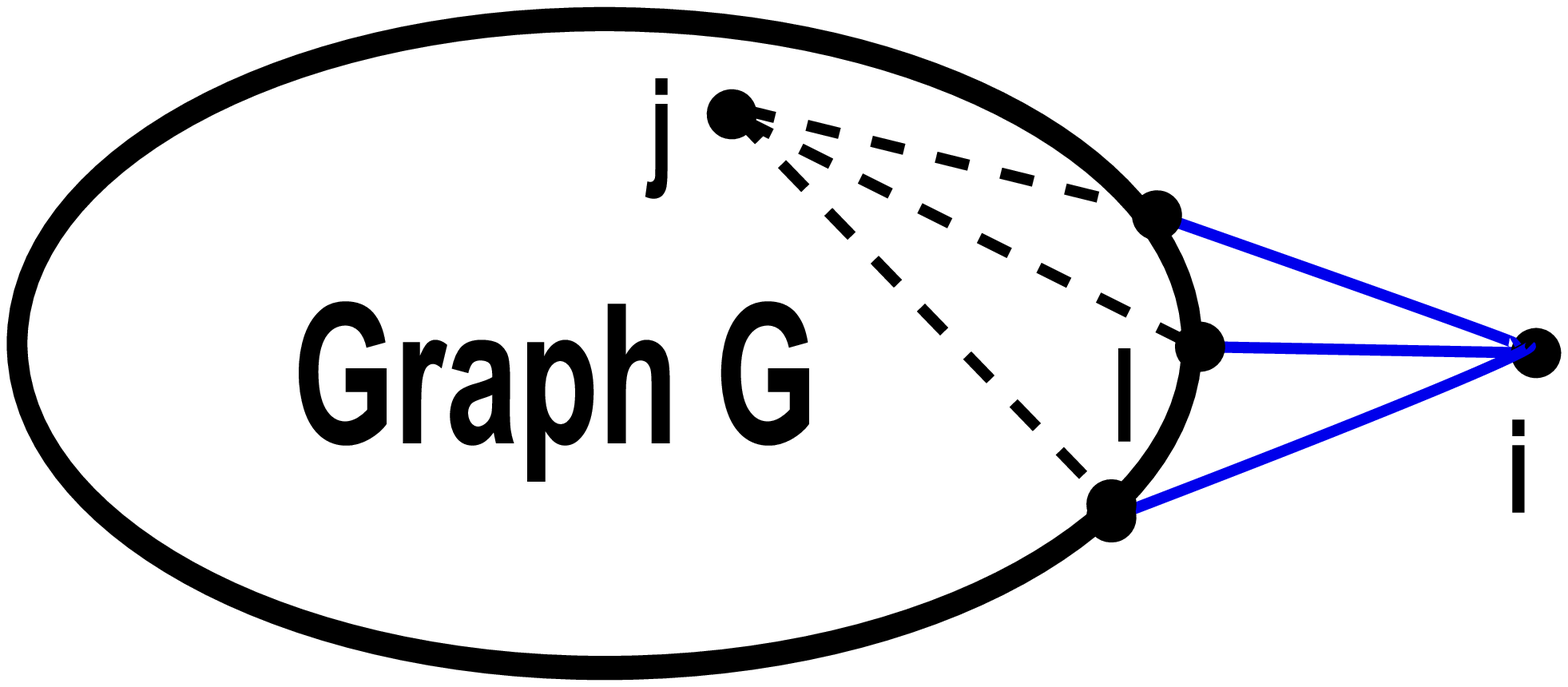}}
  \caption{Rank 1 and rank k perturbation}
  \label{fig:rank}
\end{figure}

\begin{enumerate}
\item Rank one perturbation corresponds to the situation when a new node connects with one other node in the existing graph.
\item Rank $k$ perturbation deals with the situation when the new node has $k$ neighbors in the existing graph.
\end{enumerate}
The term {\it rank} is used as it corresponds to the rank of the perturbation matrix $\Delta L$.

\subsection{Rank one perturbation}
\begin{proposition}
\label{prop3}
Let $i$ be a new node connected by one edge to an existing node $l$ in the graph $G$. Let $w_{il}$ be the weight of the new edge. Let $j$ be an arbitrary node in the graph $G$. Then
\begin{equation}
\label{equa15}
c_{ij}\approx c_{lj}^{old} + \frac{V_G}{w_{il}}
\end{equation}
where `old' represents the CTD in graph $G$ before adding $i$.
\end{proposition}

\begin{proof} (Sketch)
Since the random walk needs to pass $l$ before reaching $j$, the commute distance from $i$ to $j$ is:
\begin{equation}
\label{equa12}
c_{ij}=c_{il} + c_{lj}.
\end{equation}
It is known that:
\begin{equation}
\label{equa13}
c_{il}=\frac{(V_G + 2w_{il})}{w_{il}}
\end{equation}
where $V_G$ is volume of graph $G$ \cite{khoa2010}. We also know $c_{lj}=h_{jl} + h_{lj}$ and $h_{jl}=h_{jl}^{old}$. The only unknown factor is $h_{lj}$. By definition:
\begin{displaymath}
h_{lj}  = 1 + \sum_{q \in N(l)}p_{lq}h_{qj}
        = 1 + \sum_{q \in N(l), q\neq i}p_{lq}h_{qj} + p_{li}h_{ij}. \\
\end{displaymath}

Since $h_{qj} \approx h_{qj}^{old}$, $p_{lq}=(1-p_{li})p_{lq}^{old}$, and $h_{ij}=1+h_{lj}$,
\begin{displaymath}
\begin{split}
h_{lj}  &\approx 1 + \sum_{q \in N(l), q\neq i}(1-p_{li})p_{lq}^{old}h_{qj}^{old} + p_{li}(1+h_{lj}) \\
        &= 1 + (1-p_{li})\sum_{q \in N(l), q\neq i}p_{lq}^{old}h_{qj}^{old} + p_{li}(1+h_{lj}) \\
        &= 1 + (1-p_{li})(h_{lj}^{old}-1) + p_{li}(1+h_{lj}).
\end{split}
\end{displaymath}

After simplification, $h_{lj} = h_{lj}^{old} + \frac{2p_{li}}{1-p_{li}}.$

Then $c_{lj} \approx h_{jl}^{old} + h_{lj}^{old} + \frac{2p_{li}}{1-p_{li}}.$

Since there is only one edge connecting from $i$ to $G$, $i$ is likely an isolated point and thus $p_{li}\ll 1$. Then
\begin{equation}
\label{equa14}
c_{lj}\approx h_{jl}^{old} + h_{lj}^{old} = c_{lj}^{old}.
\end{equation}
As a result from equations \ref{equa12}, \ref{equa13}, and \ref{equa14}:
\begin{displaymath}
c_{ij}\approx \frac{(V_G + 2w_{il})}{w_{il}} + c_{lj}^{old} \approx c_{lj}^{old} + \frac{V_G}{w_{il}}
\end{displaymath}
\end{proof}

\subsection{Rank $k$ perturbation}
The rank $k$ perturbation analysis is more involved but the final formulation is an extension of the rank one perturbation.
\begin{proposition}
\label{prop4}
Denote $l\in G$ be one of $k$ neighbors of $i$ and $j$ be a node in $G$. The approximate commute distance between nodes $i$ and $j$ is:
\begin{equation}
\label{equa19}
c_{ij} \approx \sum_{l \in N(i)}p_{il}c_{lj}^{old} + \frac{V_G}{d_{i}}
\end{equation}
\end{proposition}

For the proof, see Appendix.

\section{Online Anomaly Detection Algorithm}
\label{iCTDalgorithm}
We return to our original motivation for computing incremental CTD. We are given a dataset $D$ which is {\it representative} of the underlying domain of interest. We want to check if a new data point is an anomaly with respect to $D$. We will use the CTD as a distance metric.

This section describes an online anomaly detection system using the incremental update of the eigensystem of the Laplacian in Section \ref{chapter:incrementalLaplacian} and the incremental estimation of commute time in Section \ref{chapter:incrementalCD}.

Generally CTD is robust against small changes or perturbation in data. Therefore, only the anomaly score of the new data point needs to be
estimated and be compared with the anomaly threshold in the training data. This claim will be verified by experiment in Section \ref{chapter:expres}.

\subsection{CTD-based anomaly detection}
This section reviews the batch method based on CTD to detect anomalies \cite{khoa2010}. The method is described in Algorithm \ref{CTDalgorithm}. First, a mutual $k_1$-nearest neighbor graph is constructed from the dataset. Then the graph Laplacian matrix $L$, its eigenvectors $V$ and eigenvalues $S$ are computed. Finally, a CTD distance-based anomaly detection with a pruning rule proposed by Bay and Schwabacher \cite{bay2003} is used to find the top $N$ anomalies. The anomaly score used is the average CTD of an observation to its $k_2$ nearest neighbors.
%The edge weights are inversely proportional to their Euclidean distances.

\begin{algorithm}[h!]
\caption{CTD-Based Anomaly Detection}
\label{CTDalgorithm}
\textbf{Input:} Data matrix \emph{X}, the numbers of nearest neighbors \emph{$k_1$} (for building the $k$-nearest neighbor graph) and \emph{$k_2$} (for estimating the anomaly score), the number of anomalies to return \emph{N} \\
\textbf{Output:} Top \emph{N} anomalies
\begin{algorithmic}[1]
\STATE Construct the mutual $k$-nearest neighbor graph from the dataset (using $k_1$)\\
\STATE Compute the graph Laplacian matrix $L$, its eigenvectors $V$ and eigenvalues $S$\\
\STATE Find top $N$ anomalies using the CTD based technique with pruning rule (using $k_2$). Each CTD query uses Equation \ref{equa3}\\
\STATE Return top $N$ anomalies
\end{algorithmic}
\end{algorithm}

{\bf Pruning Rule \cite{bay2003}:}  A data point is not an anomaly if its score (e.g. the average distance to its $k$ nearest neighbors) is less than an anomaly threshold. The threshold can be fixed or be adjusted as the score of the weakest anomaly found so far. Using the pruning rule, many non-anomalies can be pruned without carrying out a full nearest neighbors search.

\subsection{Online Algorithms}
Algorithm \ref{iLEDalgorithm} (denote as iLED) is a method to detect anomalies online using the eigenvalues and eigenvectors of the new Laplacian matrix which is updated incrementally. Algorithm \ref{iECTalgorithm} (denote as iECT) on the other hand is a method to detect anomalies online using incremental estimation of commute time based on hitting time.

\begin{algorithm}[h!]
\caption{Online Anomaly Detection using incremental Laplacian Eigen Decomposition (iLED)}
\label{iLEDalgorithm}
\textbf{Input:} Graph \emph{G}, Laplacian matrix \emph{L}, its eigenvalues \emph{S} and eigenvectors \emph{V}, the anomaly threshold \emph{$\tau$} of the training set, and a test data point \emph{p} \\
\textbf{Output:} Determine if $p$ is an anomaly or not \\
\begin{algorithmic}[1]
\STATE Add $p$ to $G$ using the mutual nearest neighbor graph, we have a new graph $G_n$  \\
\STATE Incrementally compute the new eigenvalues and eigenvectors of the new Laplacian $L_n$ using Algorithm \ref{iterativealgorithm}\\
\STATE Use Gram-Schmidt process \cite{golub1996} to orthogonalize the new eigenvectors \\
\STATE Determine if $p$ is an anomaly or not by estimating its anomaly score using CTDs derived from the new eigenpairs. Use pruning rule with threshold $\tau$ to reduce the computation \\
\STATE Return whether $p$ is an anomaly or not
\end{algorithmic}
\end{algorithm}

\begin{algorithm}[h!]
\caption{Online Anomaly Detection using the incremental Estimation of Commute Time (iECT)}
\label{iECTalgorithm}
\textbf{Input:} Graph \emph{G}, Laplacian matrix \emph{L}, its eigenvalues \emph{S} and eigenvectors \emph{V}, the anomaly threshold \emph{$\tau$} of the training set, and a test data point \emph{p} \\
\textbf{Output:} Determine if $p$ is an anomaly or not \\
\begin{algorithmic}[1]
\STATE Add $p$ to $G$ using the mutual nearest neighbor graph, we have a new graph $G_n$  \\
\STATE Determine if $p$ is an anomaly or not by estimating its anomaly score using incremental CTDs mentioned in Section \ref{chapter:incrementalCD}. Use pruning rule with threshold $\tau$ to reduce the computation \\
\STATE Return whether $p$ is an anomaly or not
\end{algorithmic}
\end{algorithm}

When a new data point $p$ arrives, it is connected to graph $G$ created in the training phase. The CTDs are incrementally updated to estimate the anomaly score of $p$ using the approach in sections \ref{chapter:incrementalLaplacian} and \ref{chapter:incrementalCD}. For iLED algorithm, after updating the eigenvalues and eigenvectors, we use Gram-Schmidt process \cite{golub1996} to normalize and orthogonalize the eigenvectors.

\subsection{Analysis}
First, we analyse the incremental eigen decomposition of the Laplacian in Section \ref{chapter:incrementalLaplacian}. Here $n$ is the size of the original graph (Laplacian) and $N$ is the neighborhood size used in $K_N$ (i.e., cardinality of $N_i$). Note $N \ll n$.

It takes constant time to update $\Delta \lambda$ and $O(N^2n)$ to compute $X=K_N^{\text{T}}K_N$, $O(N^3)$ for $X^{-1}$, $O(Nn)$ for $y=K_N^{\text{T}}h$ and $O(N^2)$ for $\Delta v=X^{-1}y$. Since $N \ll n$, we obtain $O(n)$ time for the incremental update of eigenvalues and eigenvectors of the Laplacian.

On the other hand, incremental estimation of commute time update in Section \ref{chapter:incrementalCD} requires $O(m)$ for each query of $c_{lj}^{old}$ where $m$ is the number of eigenvectors used. So if there are $k$ edges added to the graph, it takes $O(km)$ for each query of CTD.

Since we only need to compute the anomaly score of a test data point using the pruning rule with the threshold of anomaly score in the training set, it takes only $O(k_2)$ nearest neighbor search to determine if the test point is an anomaly or not where $k_2$ is the number of nearest neighbors for estimating the anomaly score. For each CTD query, it takes $O(m)$ for Algorithm \ref{iLEDalgorithm} (iLED), and $O(km)$ for Algorithm \ref{iECTalgorithm} (iECT). Therefore, iLED takes $O(n+k_2m)=O(n)$, and iECT takes $O(k_2km)=O(1)$ to determine if a new arriving point is an anomaly or not. Note that since $L$ and $K$ are sparse, we can get better than $O(n)$ for iLED.

\section{Experiments and Results}
\label{chapter:expres}
We report on the experiments carried out to determine and compare the effectiveness of the iECT and iLED methods. To recall, iECT uses the recursive definition of hitting time to calculate CTD while iLED uses the Laplacian definition.

\subsubsection*{Approach} We split a data set into two parts: training and test. We use Algorithm \ref{CTDalgorithm} to compute the top $N$ anomalies in the training set and use the average distance of a data point to its $k_{2}$ nearest neighbor (in CTD) as its anomaly score. The weakest anomaly in the top $N$ set is one which has the smallest average distance to its nearest neighbors and is used as the threshold value $\tau$. Then the anomaly score of each instance $p$ in the test set is calculated based on its $k_{2}$ neighbors in the training set. If this score is greater than $\tau$ then the test instance is reported as an anomaly. During the time searching for the nearest neighbors of $p$, if its average distance to the nearest neighbors found so far is smaller than $\tau$, we can stop the search as $p$ is not anomaly (pruning rule).

\subsubsection*{Data and Parameters}
The experiments were carried out on synthetic as well as real datasets. We chose the number of nearest neighbors $k_1=10$ to build the mutual nearest neighbor graph, $k_2=20$ to estimate the anomaly score, the number of Laplacian eigenvectors $m=50$. In Algorithm \ref{iterativealgorithm}, the threshold to estimate the change of $\Delta \lambda$ was $10^{-6}$ and the maximum iterations was 5. The choice of parameters was determined from the experiments. In all experiments, the batch method was used as the benchmark. The anomaly threshold $\tau$ was set based on the training data. It was the score of the weakest anomaly in the top $N=50$ anomalies found by Algorithm \ref{CTDalgorithm} in the training set.

\subsection{Synthetic datasets}
We created six synthetic datasets, each of which contained several clusters generated from Normal distributions and a number of random points generated from uniform distribution which might be anomalies. The number of clusters, the sizes, and the locations of the clusters were also chosen randomly. Each dataset was divided into a training set and a test set. There were 100 data points in every test set and half of them were random anomalies mentioned above. % $N$ was chosen to be fifty. The reason that $max(0.01n,50)$ was chosen was there were $max(0.01n,50)$ anomalies randomly generated in each dataset.

\subsubsection*{Experiments on Robustness}
We first tested the robustness of CTD between nodes in an existing set when a new data instance is introduced. As $CTD(i,j)$ between nodes $i$ and $j$ is a measure of expected path distance, the hypothesis is that the addition of a new point will have minimal influence on $CTD(i,j)$ and thus the anomaly scores of data points in the existing set are relatively unchanged.

Table \ref{tab:table2} shows the average, standard deviation, minimum, and maximum of anomaly scores of points in graph $G$ before and after a new data point was added to $G$. Graph $G$ was created from the 1,000 point dataset in Figure \ref{fig:dataset}. The result with test point was averaged over 100 test points in the test set. The result shows that the anomaly scores of data instances in $G$ do not change much. This shows CTD is a robust measure, a small change or perturbation in the data will not result in large changes in CTD. Therefore, only the anomaly score of the new point needs to be estimated.

\begin{table}[t]
  \centering
  \caption{Robustness of CTD. The anomaly scores of data instances in existing graph $G$ are relatively unchanged when a new point is added to $G$.}
  \begin{tabular}{|l|r|r|r|r|}
    \hline
                        & Average    & Std       & Min   & Max \\
    \hline
    Without test point  & 10,560.61  & 65,023.18 & 6.13  & 1,104,648.09 \\
    \hline
    With test point     & 10,517.94  & 64,840.39 & 6.13  & 1,101,169.36 \\
   \hline
  \end{tabular}
  \label{tab:table2}
\end{table}

%\begin{figure}[h]
%  \centering
%  \subfloat[Normal test point]{\label{fig:temp}\includegraphics[width=0.25\textwidth]{score_train_test_1k_p713.eps}}
%  \subfloat[Anomalous test point]{\label{fig:temp}\includegraphics[width=0.25\textwidth]{score_train_test_1k_p985.eps}}
%  \caption{Anomaly scores of data points in the training set remain relatively unchanged when adding a new data point to the training graph.}
%  \label{fig:outlierscore}
%\end{figure}

\begin{figure}[h]
	\centering
    \includegraphics[width=0.4\textwidth]{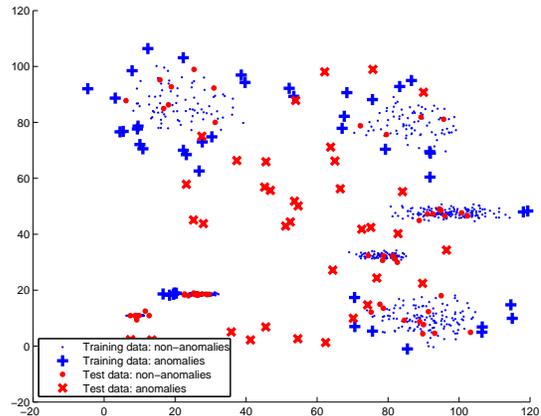}
	\caption{1,000 points dataset with training and test sets}
	\label{fig:dataset}
\end{figure}

\subsubsection*{Experiments on Effectiveness}
We applied iECT and iLED to all the datasets. The effectiveness of the iECT algorithm over iLED is shown in Figure \ref{fig:iLED_iECT_accuracy}. There were 36 anomalies detected by batch method using CTD which are shown in Figure \ref{fig:dataset}. iECT captured all of them and had 8 false positives whose scores were close to the threshold. iLED approach, on the other hand, had better precision with no false positives but worse recall with only 15 anomalies found. The reason was anomalies have more effect on the eigenvectors and eigenvalues of the graph Laplacian and thus iLED was unable to capture many of them. % Note that pruning was not used in Figure \ref{fig:iLED_iECT_accuracy} since we wanted to show the exact anomaly scores.

\begin{figure}[h]
	\centering
    \includegraphics[width=0.4\textwidth]{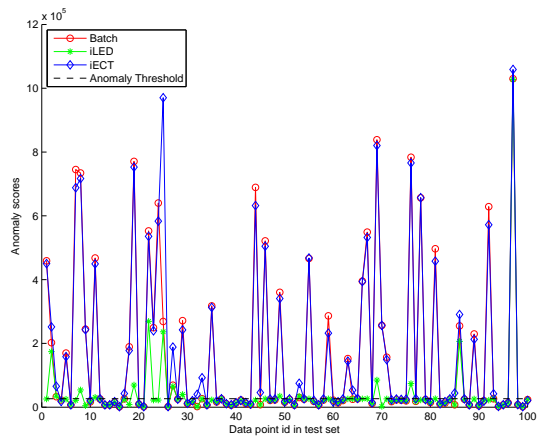}
	\caption{Accuracy of iLED and iECT in 1,000 points dataset. iECT detects anomalies better than iLED.}
	\label{fig:iLED_iECT_accuracy}
\end{figure}

To get a better understanding, Figure \ref{fig:iL} shows the eigenvalues and eigenvectors of the new graph computed using batch method and iLED in 1,000 points dataset where the test point was an anomaly and a non-anomaly. The eigenvalue graph shows the top 50 smallest eigenvalues and the eigenvector graph shows the dot product of the top 50 smallest eigenvectors of the new graph Laplacian for both batch and iLED methods. When a new data point was a non-anomaly, the approximate eigenvalues and eigenvectors were significantly more accurate than those when a new point was an anomaly.

\begin{figure}[h]
  \centering
  \subfloat[Normal point - Eigenvalues]{\label{fig:temp}\includegraphics[width=0.25\textwidth]{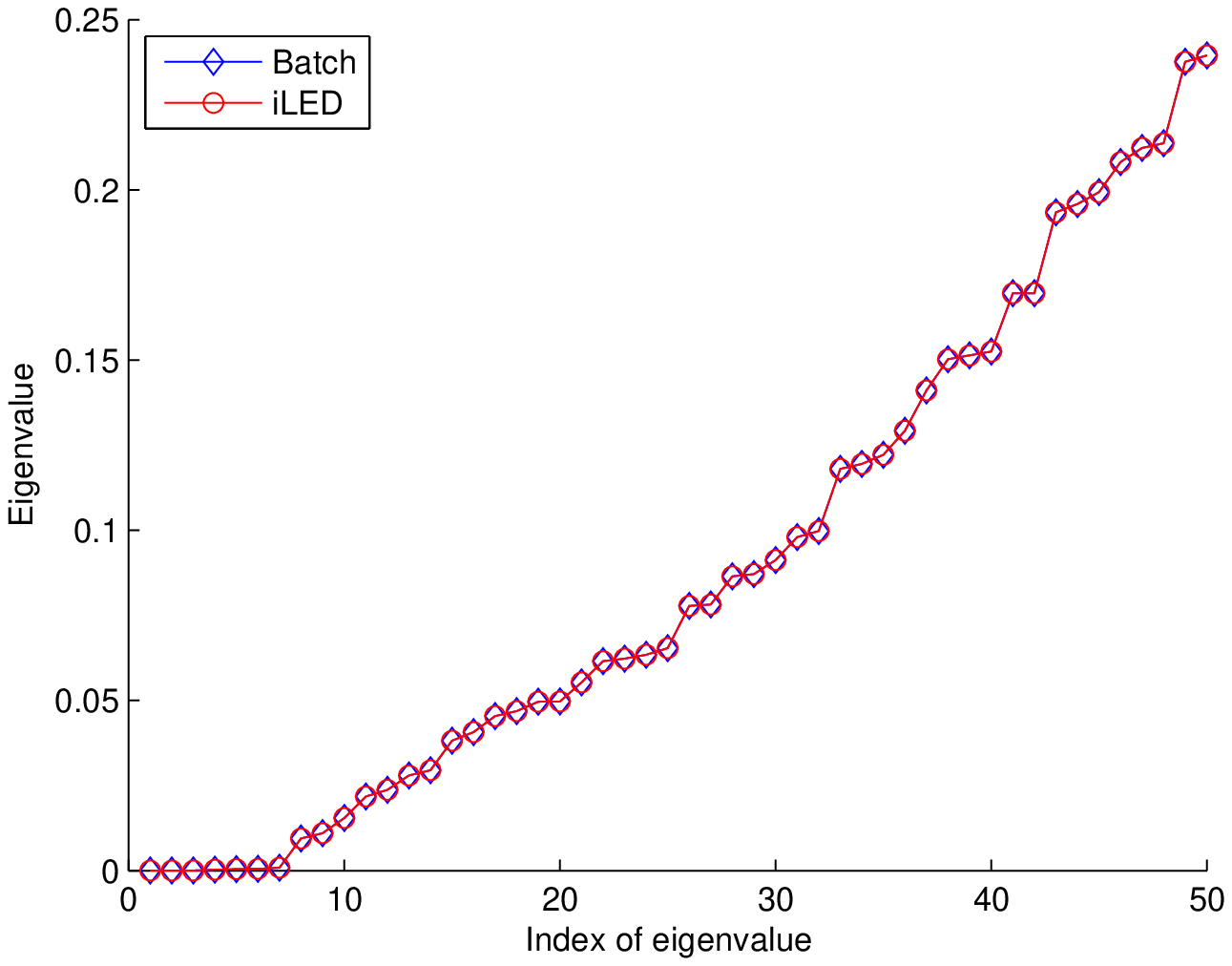}}
  \subfloat[Normal point - Eigenvectors]{\label{fig:temp}\includegraphics[width=0.25\textwidth]{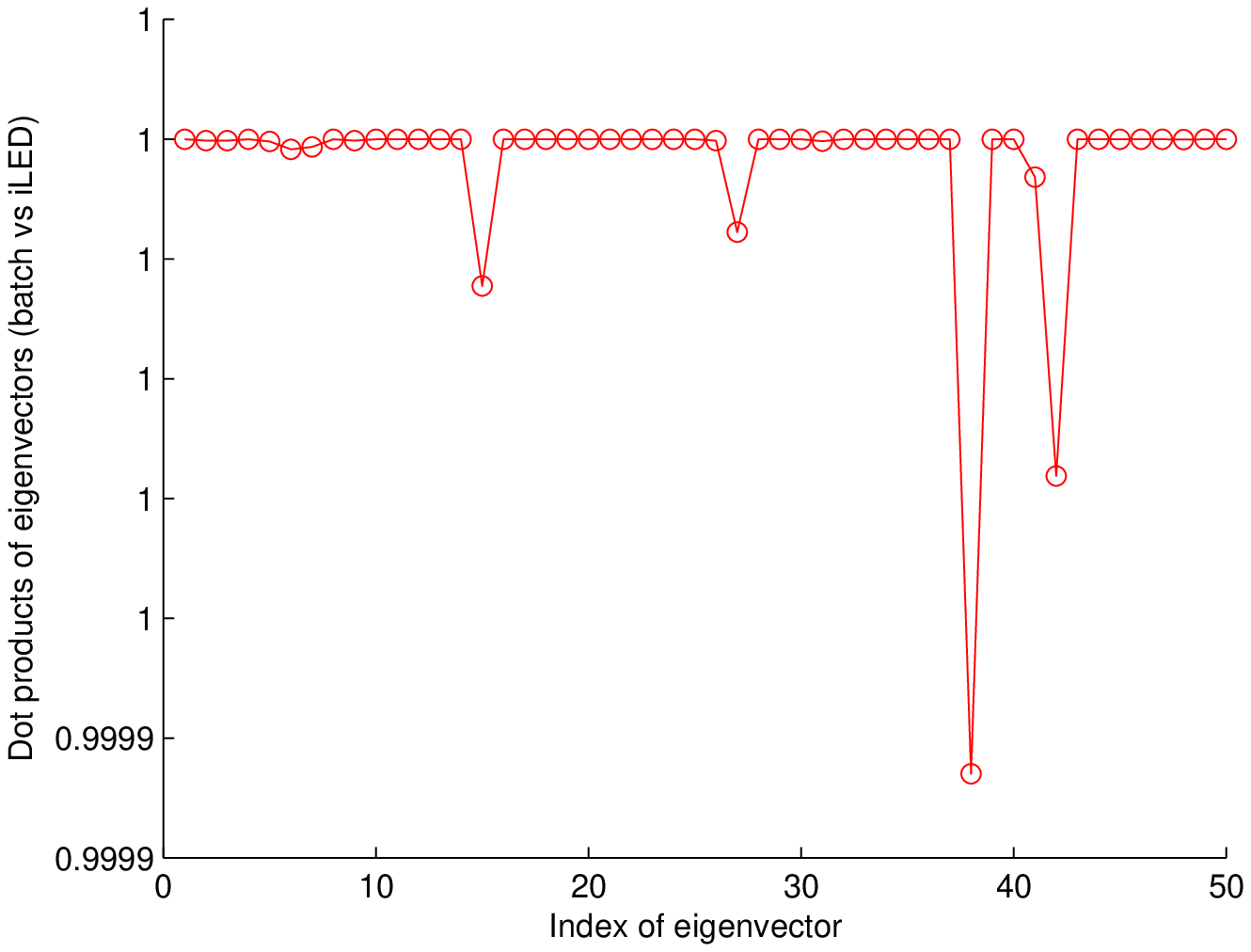}} \\
  \subfloat[Anomaly - Eigenvalues]{\label{fig:temp}\includegraphics[width=0.25\textwidth]{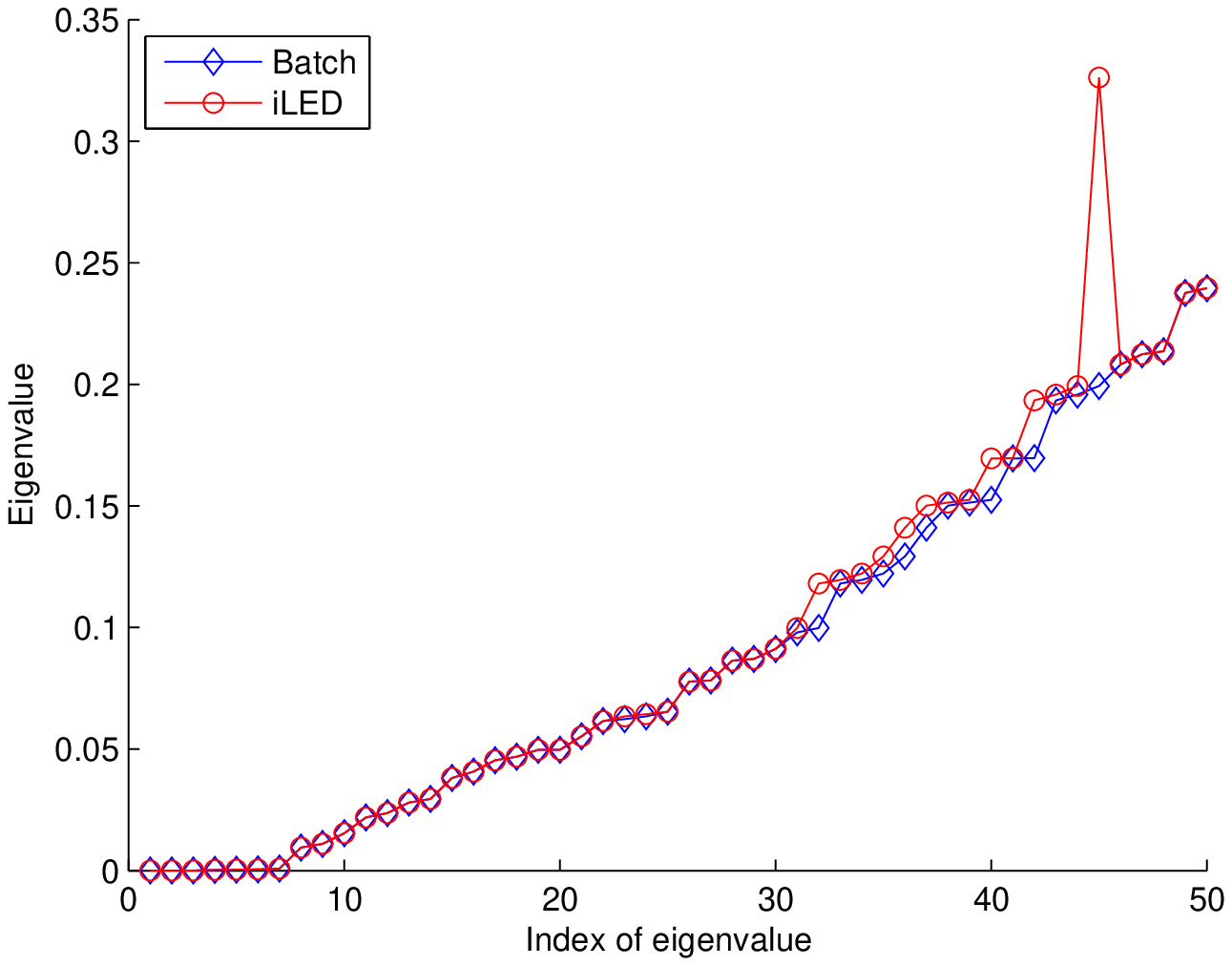}}
  \subfloat[Anomaly - Eigenvectors]{\label{fig:temp}\includegraphics[width=0.25\textwidth]{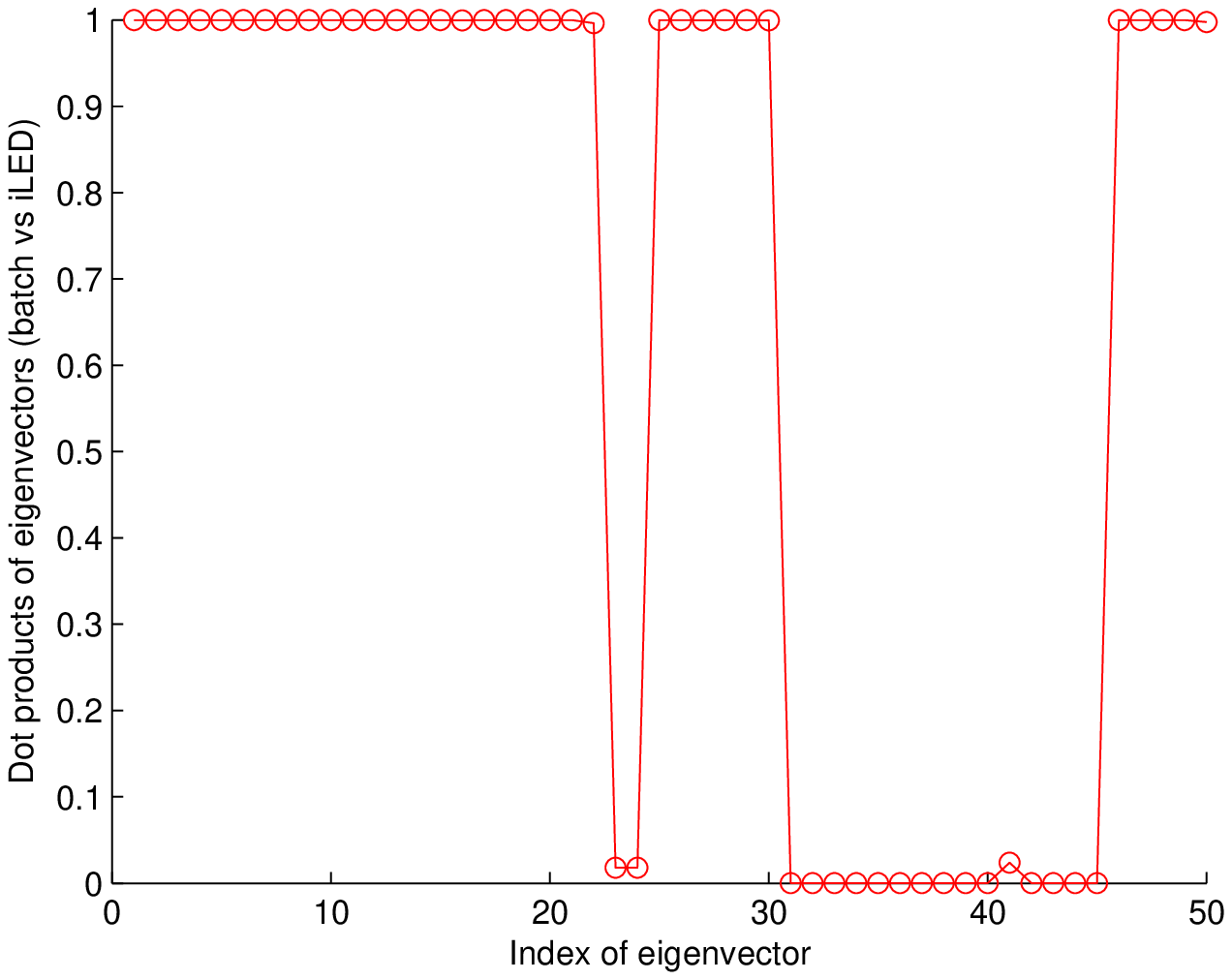}}
  \caption{Accuracy of iLED method. iLED has better approximation when a test point is not an anomaly.}
  \label{fig:iL}
\end{figure}

Table \ref{tab:table1} shows the results in accuracy and performance of iLED and iECT in the synthetic datasets. Average score was the average anomaly score over 100 test points. The precision and recall were for the anomalous class. The time was the average time to process each of 100 test points. iECT generally had better approximation than iLED, did not miss any anomaly and had acceptable false alarms. iLED, on the other hand, did not have any false alarms but did miss many anomalies. Both of them were more efficient than the batch method. Note that the scores shown here were the anomaly scores with pruning rule and the scores for anomalies are always much higher than scores for normal points. Therefore the scores were dominated by the scores of anomalies and that is the reason why iLED had much lower scores. In fact, iLED is more accurate than iECT in estimating the scores of normal points.

%\begin{figure*}
%\centering
%\epsfig{file=rank1.eps}
%\caption{A sample black and white graphic (.eps format)
%that needs to span two columns of text.}
%\end{figure*}

\begin{table*}[t]
  \centering
  \caption{Effectiveness of incremental methods. iECT generally does not miss any anomaly and has acceptable false alarms. iLED, on the other hand, does not have any false alarms but misses many anomalies.} %Both of them are more efficient than the batch method.}
  \begin{tabular}{|r|r|r|r|r|r|r|r|r|r|r|}
    \hline
    Dataset & \multicolumn{4}{|c|}{iLED} & \multicolumn{4}{|c|}{iECT} & \multicolumn{2}{|c|}{Batch}\\
    \hline
    Size & Avg Score & Precision (\%) & Recall (\%) & Time (s) & Avg Score & Precision (\%) & Recall (\%) & Time (s) & Avg Score & Time (s) \\
    \hline
     1,000  & $3.84 \times 10^4$    & 100   & 41.7  & 0.13 & $1.69 \times 10^5$ & 81.8   & 100   & 0.20 & $1.61 \times 10^5$    & 0.17 \\
    10,000  & $7.09 \times 10^5$    & 100   & 53.2  & 1.28 & $4.87 \times 10^6$ & 95.9   & 100   & 1.42 & $4.99 \times 10^6$    & 2.04 \\
    20,000  & $5.36 \times 10^6$    & 100   & 83.3  & 2.64 & $1.75 \times 10^7$ & 80.0   & 100   & 2.96 & $1.70 \times 10^7$    & 4.53 \\
    30,000  & $4.81 \times 10^6$    & 100   & 39.6  & 3.68 & $1.39 \times 10^8$ & 96.0   & 100   & 4.33 & $1.40 \times 10^8$    & 7.13 \\
    40,000  & $3.27 \times 10^6$    & 100   & 15.6  & 4.44 & $5.17 \times 10^7$ & 71.1   & 100   & 5.19 & $4.88 \times 10^7$    & 9.05 \\
    50,000  & $8.88 \times 10^6$    & 100   & 32.5  & 5.70 & $6.15 \times 10^7$ & 87.0   & 100   & 6.61 & $5.96 \times 10^7$    & 11.60 \\
   \hline
  \end{tabular}
  \label{tab:table1}
\end{table*}

There is an interesting dynamic at play between the anomaly, pruning rule, iECT, iLED, and the number of anomalies in the data. iECT was slightly slower than iLED in the experiment. The reason is we have many anomalies in the test set. We know that the pruning rule only works for non-anomalies. Moreover, iLED is faster per CTD query compared to iECT. Therefore, for anomalies, iLED is generally faster. Furthermore, because iLED tends to underestimate the scores of anomalies (and that was the reason it missed some anomalies), anomalies are treated as non-anomalies and the pruning kicks-in making it faster. In practise, since most of the test points are not anomalies, iECT will be more efficient than iLED. It is shown in Figure \ref{fig:iLED_iECT_time} where except a few false alarms, iECT was generally faster than iLED in 50,000 points dataset when test instances were not anomalies. The same tendency also happened in other datasets used in the experiments.

\begin{figure}[t]
	\centering
    \includegraphics[width=0.4\textwidth]{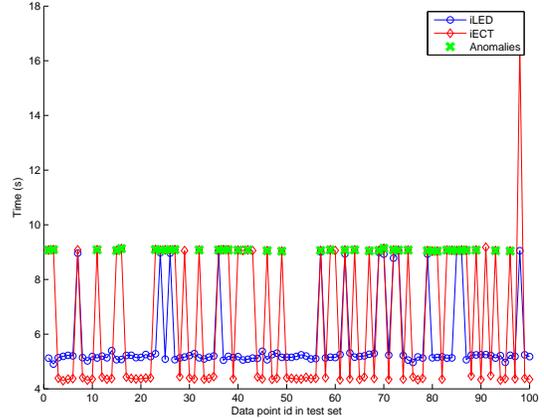}
	\caption{Performance of iLED and iECT in 50,000 points dataset. Except a few false alarms, iECT is generally faster than ilED when a test point is not an anomaly.}
	\label{fig:iLED_iECT_time}
\end{figure}

\subsection{Real Datasets}
\subsubsection{DBLP dataset}
In this section, we evaluated the iECT method on the DBLP co-authorship network. Nodes are authors and edge weights are the number of collaborated papers between the authors. Since the graph is not fully connected, we extracted its biggest component. It has 344,800 nodes and 1,158,824 edges.

We randomly chose a test set of 50 nodes and removed them from the graph. We ensured that the graph remained connected. After training, each node was added back into the graph along with their associated edges.

Since the size of the graph is very large, normal training using the batch mode in Algorithm \ref{CTDalgorithm} is not feasible. Instead we implemented the approximate method proposed by Spielman and Srivastava (SS) \cite{spielman2008} and used the underlying linear time CMG solver proposed by Koutis \cite{koutis2009}. The SS methods combines random projections with a linear time solver for diagonally dominant matrices to approximate the CTD. The SS method creates a matrix $Z$ from which CTD between two nodes can be computed in $k=O(\log n)$ time with provable accuracy. In practice we can use $k$ to be much smaller than $O(\log n)$ and still attain highly accurate results.

We trained the graph using the SS approach, stored the matrix $Z$ and used $Z$ to query the $c_{lj}^{old}$ in iECT algorithm. The batch method here is the CTD approximation using the matrix $Z_{new}$ created from the new graph after adding each test data point. The parameter for random projection was $k=200$.

The result shows that it took 0.0066 seconds on average over 50 test data points to detect whether each test point was an anomaly or not. The batch method, which is the fastest approximation of CTD to date, required 944 seconds on average to process each test data point. This dramatically highlights the constant time complexity of iECT algorithm and suggests that iECT is highly suitable for the computation of CTD in an incremental fashion. Since there was no anomaly in the random test set, we cannot report the detection accuracy here. The average anomaly score over all the test points of iECT was 1.1 times higher than the batch method. This shows the relatively high accuracy of iECT approximation even in a very large graph.

\subsubsection{KDD Cup 1999 datasets}
We used the 10\% dataset from the KDD cup 1999 competition provided by UCI Machine Learning Repository \cite{FrankAsuncion2010}. It was used to build detection tools of network attacks or intrusions. Since the dataset is huge and there are more anomalies than normal instances, we sampled 2,200 data points from it where there were 2,000 normal points and 200 anomalies (network intrusions). Categorical features were ignored and 38 numerical features were used. The dataset was divided into a training set and a test set with 100 data points.

iLED and iECT were applied on this dataset and min-max scaling was used as data normalization. iLED had a precision of 100\% and a recall of 66.7\% while iECT had a precision of 75\% and a recall of 100\%. The average anomaly scores of iLED and iECT were 2\% and 1\% lower than that of the batch method, respectively.

\subsubsection{NICTA datasets}
The dataset is from a wireless mesh network which has seven nodes deployed by NICTA at the School of IT, University of Sydney \cite{zaidi2009}. It used a traffic generator to simulate traffic on the network. Packets were aggregated into one-minute time bins and the data was collected in 24 hours. There were 391 origin-destination flows and 1,270 time bins. Some anomalies were introduced to the network including DOS attacks and ping floods. The dataset was divided into a training set and a test set with 100 data points.

iLED and iECT were applied on this dataset and min-max scaling was used as data normalization. iLED had a precision of 100\% and a recall of 27.3\% while iECT had a precision of 84.6\% and a recall of 100\%. The average anomaly scores of iLED and iECT were 20\% lower and 2\% higher than that of the batch method, respectively. The tendency of the detection here of iLED and iECT are also similar to those of the synthetic datasets.

\subsection{Summary and Discussion}
The experimental results on both synthetic and real datasets show that iECT generally has better detection ability than iLED. iECT has very high recall and acceptable precision. iLED, on the other hand, has very high precision but low recall. Both of them are faster than the batch method. The results on real datasets collected from different sources also have similar tendency showing the reliability and effectiveness of the proposed methods.

The experiments also reveal that iLED tends to underestimate the CTDs for anomalies while iECT tends to overestimate the CTDs for non-anomalies. It leads to a high precision for iLED and a high recall for iECT. If we can come up with a strategy to combine the strengths of the two methods, we can have a more accurate estimation. iECT is faster than iLED but it can only be used in case where a new test point is added and cannot be used when there are weigh updates in the graph. iLED on the other hand can be used in both cases by just changing the perturbation matrix.

\section{Related work}
\label{chapter:related}
Incremental learning using an update on eigen decomposition has been studied for a long time. Early work studied the rank one modification of the symmetric eigen decomposition \cite{golub1973, bunch1978, gu1994}. The authors reduced the original problem to the eigen decomposition of a diagonal matrix. Though they can have a good approximation of the new eigenpair, they are not suitable for online applications nowsaday since they have at least $O(n^2)$ computation for the update.

More recent approach was based on the matrix perturbation theory \cite{champagne1994, agrawal2008}. They used the first order perturbation analysis of the rank-one update for a data covariance matrix to compute the new eigenpair. The algorithms have a linear time computation. The advantage of using the covariance matrix is if the perturbation involving an insertion of a new point, the size of the covariance matrix is unchanged. This approach cannot be applied directly to increasing matrix size due to an insertion of a new point. For example, in spectral clustering or CTD-based anomaly detection, the size of the graph Laplacian matrix increases when a new point is added to the graph.

Ning et. al \cite{ning2007} proposed an incremental approach for spectral clustering with application to monitor evolving blog communities. It incrementally updates the eigenvalues and eigenvectors of the graph Laplacian matrix based on a change of an edge weight on the graph using the first order error of the generalized eigen system. %Though the authors claimed that it can work for a sequence of weight changes involving the insertion the a new point, the algorithm is only suitable for weight update cases.

\section{Conclusion}
\label{chapter:conclusion}
The paper shows two novel approaches to compute CTD incrementally. The first one incrementally updates the eigenvectors and eigenvalues of the graph Laplacian matrix. It is linearly scaled and can be applied to the estimation of CTD incrementally or any application involving graph spectral computation. The second approach incrementally estimates CTD in constant time using the property of random walk and hitting time. We design novel anomaly detection algorithms using two approaches to detect anomalies online. The experimental results show the effectiveness of the proposed approaches in terms of performance and accuracy. It took less than 7 milliseconds on average to process a new arriving point in a graph of more than 300,000 nodes and one million edges. Moreover, the idea of this work can be extended in many applications which use the CTD and it is the direction for our future work.


\section*{Acknowledgment}
The authors of this paper acknowledge the financial support of the Capital Markets CRC.



%\bibliographystyle{IEEEtran}
\bibliographystyle{IEEEtran}
\bibliography{CommuteDistance_arxiv}

\appendix
We provide relation between perturbation and incidence matrix, and the proof details of Propositions \ref{prop2} and \ref{prop4}.

%\section{Detailed Proofs}
\subsection{Incidence matrix and perturbation}
It is well known that the Laplacian of a graph can be expressed in terms of an incidence matrix.

\begin{definition}
Given a weighted graph $G=(V,E,W)$ and an arbitrary but fixed orientation of the edges, the incidence matrix  $R$ is a $|V| \times |E|$ matrix where the columns of the matrix are defined as
\begin{equation}
r_{e}(w) =
    \begin{cases}
    \sqrt{w} & \text{at location $v$ if $v$ is the head of $e$} \\
    -\sqrt{w} & \text{at location $v$ if $v$ is the tail of $e$} \\
    0 & \text{otherwise} \\
    \end{cases}
\end{equation}
Note that $r_{e}$ is a column vector of size $|V|$. We can also express each $r_{e} = \sqrt{w}u_{e}$ where $u_{e}$ is a column vector with entries 1, -1 at locations corresponding to head and tail of $e$ and 0 at other locations.
\end{definition}

\begin{fact}
If $L$ is a graph Laplacian then $L=RR^{\text{T}}$ \cite{chung1997}.
\end{fact}

\begin{fact}
If  an edge $e$ undergoes a  similarity change $\Delta w_{e}$, the new graph Laplacian $L_n$ is $L_n=R_nR_n^{\text{T}}$ where $R_n$=[$R$ $r_{e}(\Delta w_{e})$] \cite{ning2007}. Therefore, a change in an edge weight can be represented by appending an  incidence vector to $R$.
\end{fact}

\begin{proposition}
\label{prop1}
Denote $R$ as an incidence matrix of a given graph $G=(V,E,W)$. Suppose a new node is added to the graph which results in $k$ new edges. Let $R_{n}$ represent the new $|V+1| \times |E+k|$ matrix and let $\Delta R$ represent the matrix of new incidence vectors of size $|V+1| \times k$. Then the new Laplacian $L_{n}$ can be expressed in terms of the old Lapalcian $L$ as
\[
L_{n} = \begin{bmatrix}L & 0 \\ 0 & 0\end{bmatrix} + \Delta L
\]
\end{proposition}

Also note that if $(\lambda,v)$ is an eigenpair of $L$ then $\left(\lambda, \left(\begin{array}{c}v \\ 0 \end{array}\right)\right)$ is an eigenpair of $\begin{bmatrix}L & 0 \\ 0 & 0\end{bmatrix}$.

%For definition of incidence matrix and the proof, see Appendix.\textbf{Proof for Proposition \ref{prop1}:}
\begin{proof}
Let $E_{n}$ be the set of new edges resulting from the addition of a new node. Then
\begin{displaymath}
\begin{split}
R_nR_n^{\text{T}}   &= \left[\begin{pmatrix}R \\ 0\end{pmatrix}\Delta R\right]\left[\begin{pmatrix}R \\ 0\end{pmatrix}\Delta R\right]^{\text{T}}  \\
                    &= \begin{bmatrix}L & 0 \\ 0 & 0\end{bmatrix} + \sum_{e \in E_{n}}w_{e}u_{e}u_{e}^{\text{T}}
                    = \begin{bmatrix}L & 0 \\ 0 & 0\end{bmatrix} + \Delta L = L_n
\end{split}
\end{displaymath}
\end{proof}

\subsection{Incremental Update of Eigenvalues and Eigenvectors}
\textbf{Proof for Proposition \ref{prop2}:}
\begin{proof}
Given the eigen decomposition of the new Laplacian matrix:
\begin{equation}
(L + \Delta L)(v + \Delta v) = (\lambda + \Delta \lambda)(v + \Delta v)
\end{equation}

The perturbation is:
\begin{equation}
\label{equa5}
\Delta L = \sum_{e \in E_{n}}w_{e}u_{e}u_{e}^{\text{T}}
\end{equation}

Since $Lv = \lambda v$,
\begin{equation}
\label{equa6}
L\Delta v + \Delta Lv + \Delta L\Delta v = \Delta \lambda v + \lambda\Delta v + \Delta \lambda\Delta v
\end{equation}

Left multiply both sides by $v^\text{T}$:
\begin{displaymath}
v^\text{T}L\Delta v + v^\text{T}\Delta Lv + v^\text{T}\Delta L\Delta v = v^\text{T}\Delta \lambda v + v^\text{T}\lambda\Delta v + v^\text{T}\Delta \lambda\Delta v
\end{displaymath}

Since $v^{\text{T}}L = \lambda v^{\text{T}}$ ($L$ is symmetric):
\begin{displaymath}
v^{\text{T}}\Delta \lambda(v + \Delta v) = v^{\text{T}}\Delta L (v + \Delta v)
\end{displaymath}

Then we have the update of the eigenvalue $\lambda$:
\begin{equation}
%\label{equa7}
\begin{split}
\Delta \lambda  &= \frac{v^{\text{T}}\Delta L (v + \Delta v)}{v^{\text{T}}(v + \Delta v)}
                = \frac{v^{\text{T}}\Delta L (v + \Delta v)}{1 + v^{\text{T}}\Delta v} \\
                &= \frac{v^{\text{T}}\sum_{e \in E_{n}}w_{e}u_{e}u_{e}^{\text{T}} (v + \Delta v)}{1 + v^{\text{T}}\Delta v} \\
                &= \frac{\sum_{e \in E_{n}}w_{e}[v(i)-v(j)][v(i)-v(j)+\Delta v(i)-\Delta v(j)]}{1 + v^{\text{T}}\Delta v} \\
\end{split}
\end{equation}

From equation \ref{equa6} we have:
\begin{displaymath}
[L + \Delta L - (\lambda + \Delta \lambda)I]\Delta v = (\Delta \lambda I - \Delta L)v
\end{displaymath}

Denotes $K= L + \Delta L - (\lambda + \Delta \lambda)I$ and $h=(\Delta \lambda I - \Delta L)v$,
we have $\Delta v = K^{-1}h.$
\end{proof}

\subsection{Rank $k$ Perturbation}
\begin{lemma}
\label{lem1}
Denote $l\in G$ is one of $k$ neighbors of $i$ and $j$ is a node in $G$. We have:
\begin{displaymath}
\sum_{l \in N(i)}p_{il}h_{li} = \frac{V_G}{d_{i}}+1.
\end{displaymath}
\end{lemma}

\begin{proof}
Using the reversibility property of the random walk, it is easy to prove that the expected number of steps that a random walk which has just visited node $i$ will take before returning back to $i$ is $graph\text{-}volume/d_{i}$ \cite{lovasz1993}.

In case of $i$, this distance equals to the distance from $i$ to one of its neighbors $l$ (one step) plus the hitting time $h_{li}$. Since the random walk goes from $i$ to $l$ with the probability $p_{il}$, we have $1 + \sum_{l \in N(i)}p_{il}h_{li} = \frac{V_G+2d_{i}}{d_{i}}$. Therefore, $\sum_{l \in N(i)}p_{il}h_{li} = \frac{V_G}{d_{i}}+1$.
\end{proof}

\vspace{10pt}
\textbf{Proof for Proposition \ref{prop4}:} %Let $l\in G$ be one of the $k$ neighbors of $i$ and $j$ is a node in $G$. The approximate commute distance between nodes $i$ and $j$ is:
%\begin{displaymath}
%%\label{equa19}
%c_{ij} \approx \sum_{l \in N(i)}p_{il}c_{lj}^{old} + \frac{V_G}{d_{i}}
%\end{displaymath}

\begin{proof} (Sketch)
By definition,
\begin{displaymath}
h_{ij} = 1 + \sum_{l \in N(i)}p_{il}h_{lj} = 1 + \sum_{l \in N(i)}p_{il}(1 + \sum_{q \in N(l)}p_{lq}h_{qj})
\end{displaymath}

Using the same approach as the rank one case,
\begin{displaymath}
\begin{split}
h_{ij}  &= 1 + \sum_{l \in N(i)}p_{il}[1 + (1-p_{li})\sum_{q \in N(l), q\neq i}p_{lq}^{old}h_{qj}^{old} + p_{li}h_{ij}] \\
        &= 1 + \sum_{l \in N(i)}p_{il}[1 + (1-p_{li})(h_{lj}^{old}-1) + p_{li}h_{ij}]
\end{split}
\end{displaymath}

After a few manipulations, we have
\begin{displaymath}
h_{ij} = \frac{1 + \sum_{l \in N(i)}p_{il}h_{lj}^{old}-\sum_{l \in N(i)}p_{il}p_{li}h_{lj}^{old}+\sum_{l \in N(i)}p_{il}p_{li}}{1-\sum_{l \in N(i)}p_{il}p_{li}}.
\end{displaymath}

Because $\sum_{l \in N(i)}p_{il}p_{li} \ll 1$,
\begin{equation}
\label{equa16}
h_{ij} \approx 1 + \sum_{l \in N(i)}p_{il}h_{lj}^{old}.
\end{equation}

Since the commute distance between two nodes is the average of all possible path-length between them, $h_{ji} \approx \frac{1}{k}\sum_{l \in N(i)}(h_{jl} + h_{li})$. Instead of using the normal average, we take into account the probability $p_{il}$:
\begin{equation}
\label{equa17}
h_{ji} \approx \sum_{l \in N(i)}p_{il}(h_{jl} + h_{li}) = \sum_{l \in N(i)}p_{il}h_{jl} + \sum_{l \in N(i)}p_{il}h_{li}
\end{equation}

We have $h_{jl}\approx h_{jl}^{old}$. Moreover, from Lemma \ref{lem1} we have $\sum_{l \in N(i)}p_{il}h_{li} = \frac{V_G}{d_{i}}+1$.
Then from \ref{equa17},
\begin{equation}
\label{equa18}
h_{ji} \approx \sum_{l \in N(i)}p_{il}h_{jl}^{old} + \frac{V_G}{d_{i}}+1
\end{equation}

As a result of equations \ref{equa16} and \ref{equa18},
\begin{displaymath}
\begin{split}
c_{ij}  &\approx 1+ \sum_{l \in N(i)}p_{il}c_{lj}^{old} + \frac{V_G}{d_{i}} +1
        \approx \sum_{l \in N(i)}p_{il}c_{lj}^{old} + \frac{V_G}{d_{i}}
\end{split}
\end{displaymath}

When $k=1$ (rank one case), the formula becomes Equation \ref{equa15}.
\end{proof}

\end{document}